\documentclass[conference]{IEEEtran}
\IEEEoverridecommandlockouts
\usepackage{cite}
\usepackage{amsmath,amssymb,amsfonts}
\usepackage{graphicx}
\usepackage{booktabs}
\usepackage{xcolor}
\usepackage{url}
\usepackage[hidelinks]{hyperref}

\newtheorem{proposition}{Proposition}
\newtheorem{corollary}{Corollary}

\begin{document}

\title{Memory for Attention:\\ Language-Conditioned Re-Perception with a Vision--Language--Motion Map}
\author{Dibyendu Ghosh}
\maketitle

\begin{abstract}
A robot carrying a persistent, behavior-annotated map faces two very different planning questions, and
its memory answers only one of them well. The \emph{spatial-navigation} question -- how to walk around a
room -- we address first, and report a negative: building on Vision--Language--Motion Maps (VLMM),
a behavior-aware planner cost cuts a planning-time objective by $\sim$35\% over
$28$ AI2-THOR scenes, but under closed-loop execution the benefit nearly vanishes ($\sim$4\%) and an
on-demand vision--language model (VLM) does as well. The \emph{resource-allocation} question is
different: under a limited perception budget, \emph{what should the robot attend to right now} to keep
its own map fresh? Framing re-perception as this attention decision, we show a persistent map's
\emph{memory} -- change-history, or even just recency of last sighting -- yields the best of the re-perception
schedules we test (held-out), while the memoryless category prior is the weakest of them under the
$\sqrt{\cdot}$-law schedule --- though under a Whittle index it leads memory until heterogeneity is
real. The gain grows with per-instance heterogeneity as a Cauchy--Schwarz bound predicts, tracking
$\mathrm{Var}(\sqrt\lambda)$, the variance of root-volatility, and reallocates budget toward the
important objects the schedule protects; against a \emph{real} CLIP movability prior it is
$+21$--$26\%$, of which roughly half survives once that prior's saturated scale is calibrated
($+7$--$13\%$). The map's \emph{full} combination earns its keep when the task is
\emph{language-conditioned}: told what to keep track of, VLMM grounds the relevant objects
(open-vocabulary) and tracks their change (memory), beating a relevance-weighted recency baseline
($+2.9\%$ over $26$ queries, at full heterogeneity; the ordering reverses when instance rates track
category norms) -- and a category prior
($+9.2\%$). The map earns its keep not by telling the robot how to walk around a room, but by telling it
what to pay attention to.
\end{abstract}

\section{Introduction}
The environments a robot operates in rarely hold still: objects are moved, doors open, things are set
down and picked up. A map that records only \emph{what} and \emph{where} things are is therefore
perpetually going out of date; to act well, a robot benefits from also knowing \emph{how} the scene
behaves. Open-vocabulary maps ground language to navigation goals -- VLMaps~\cite{huang2023vlmaps}
localises language-specified landmarks, and ConceptFusion~\cite{jatavallabhula2023conceptfusion},
ConceptGraphs~\cite{gu2024conceptgraphs}, and HOV-SG~\cite{werby2024hovsg} attach features to voxels or
object graphs -- but assume a \emph{static} world. A separate line adds behavior:
Khronos~\cite{schmid2024khronos} builds a spatio-temporal map, and Dewan~\emph{et
al.}~\cite{dewan2017deep} classify LiDAR points as non-movable, movable, or dynamic. Vision--Language--
Motion Maps (VLMM)~\cite{vlmm2026} combine both, annotating each element with an open-vocabulary
feature, a \emph{movability prior} $\rho\in[0,1]$, an \emph{observed} cross-frame motion score $o$, and
a per-element \emph{confidence} $\gamma$.

Given such a map, what is it \emph{for}? The obvious use is \emph{spatial}: feed the behavior fields to
the planner so the robot routes around what might move. We pursue this first and find it wanting. One
\emph{can} encode the map into a planner cost -- a route is good not only if short but if it
\emph{stays valid}, avoiding regions where objects churn -- and on a planning-time objective it helps
substantially. But under closed-loop execution the benefit largely washes out (\S\ref{sec:spatial}):
low-level control and reactive replanning already handle a mover that appears
ahead~\cite{marder2010office}, so shaping the \emph{global} path around \emph{possible} motion buys
little, and an on-demand query does as well as the map's memory. Answering ``how do I walk around this
room?'' is not where a behavior map earns its keep.

Its distinctive value lies in a different kind of planning. A robot's perception is \emph{budgeted} --
it cannot re-observe every object at every step, whether because sensing, travel, or computation is
limited. It must therefore decide \emph{what to attend to}: which map entries to re-check to keep the
representation fresh as the scene changes. This is a \emph{resource-allocation} decision, answering
``what should I pay attention to right now?'' rather than ``how do I move?'', and it is precisely where
a persistent map's \emph{memory} of what has actually been changing should help -- and where an
on-demand vision--language model (VLM), which reports only category-level movability, should not.
Dynamic and lifelong maps~\cite{dynamem2024,dovsg2025,dualmap2025} re-perceive and re-map, but treat
\emph{when} to look as a fixed schedule or a coverage sweep; we ask instead how the map's own history
should \emph{direct} that attention.

Scene-graph memory can even \emph{predict} where change occurs from accumulated
observation~\cite{kurenkov2023scenegraphmemory}, yet existing dynamic maps still take \emph{when} and
\emph{which} to re-observe as given. Deciding \emph{where to look} is studied in active
perception~\cite{bajcsy1988active,bajcsy2018revisiting}, next-best-view~\cite{connolly1985nbv,
bircher2016receding} and persistent monitoring~\cite{smith2012persistent,alamdari2014persistent} -- but
from geometric coverage or \emph{known} dynamics. The refresh problem itself is classic in web
crawling~\cite{cho2003refresh} (a value-weighted $\sqrt{\cdot}$-law optimum), a restless
bandit~\cite{whittle1988restless} whose Whittle index we adopt, and underlies age-of-information
scheduling~\cite{kadota2018scheduling,yates2021aoi} and change-rate
estimation~\cite{avrachenkov2020changerate}. Taking VLMM as the map, we drive the schedule from its own
observed per-instance change history and ask when that beats an on-demand category prior.

\smallskip\noindent\textbf{Contributions.} (i) A distinction between two uses of a behavior-annotated
map: \emph{spatial-navigation} path-shaping, which a closed-loop evaluation shows is marginal
and needs no memory, and \emph{resource-allocation} attention, where memory pays off. (ii) A formulation
of budgeted re-perception as an \emph{attention} decision -- value-weighted staleness minimised by a
$\sqrt{\cdot}$-law schedule, a Cauchy--Schwarz bound predicting \emph{when} the map's memory helps, and
a Whittle-index policy as a proper reference. (iii) An evaluation over $951$ objects and competitive
schedulers showing observed history gives the best schedule -- matching the $\sqrt{\cdot}$-law oracle and
beating age-based scheduling where importance is skewed (recency is the stronger arm when it is
not), reallocating budget toward the important objects,
and holding under a real CLIP prior, observation noise, and a downstream fetch task; the gain tracks the
root-volatility variance $\mathrm{Var}(\sqrt\lambda)$. (iv) The map's \emph{distinctive} use --
\emph{language-conditioned} re-perception: a spoken instruction sets which objects to keep fresh
(open-vocabulary relevance) and the schedule tracks their change (dynamics); an ablation shows a map
with \emph{both} channels beats a static open-vocabulary map, a category prior, and -- crucially -- a
relevance-weighted recency heuristic, so the motion channel adds value \emph{beyond a last-seen
timestamp} and the use does not collapse to ``open-vocabulary map $+$ recency''; dropping either channel
fails.

\section{Methodology}\label{sec:formulation}
\textbf{The attention decision.} Fig.~\ref{fig:method} gives the pipeline. A map holds $N$ elements.
Element $i$ changes (is moved) as a Poisson process of rate $\lambda_i$; if re-observed every $\tau_i$ steps, its time-average \emph{staleness}
(probability the map entry is out of date) is
\begin{equation}\label{eq:stale}
g(\lambda_i\tau_i)=1-\tfrac{1-e^{-\lambda_i\tau_i}}{\lambda_i\tau_i}\approx\tfrac12\lambda_i\tau_i,
\qquad \lambda_i\tau_i\!\ll\!1 .
\end{equation}
Not all entries matter equally: element $i$ carries an importance $w_i$ (a fragile or task-relevant
object costs more when stale). Under a perception budget of $K$ re-observations per step, the attention
policy chooses re-check frequencies $f_i=1/\tau_i$ with $\sum_i f_i=K$ to minimise value-weighted
staleness $S=\sum_i w_i\,g(\lambda_i\tau_i)$. The simulator draws a change
per step with probability $\lambda_i$ (a discrete-time hazard); \eqref{eq:stale} is its continuous
Poisson approximation, which at our operating point overstates the change probability by $1$--$2\%$,
identically for every arm. Our experiments run at
$\lambda_i\tau_i\!\approx\!1$, outside the regime this linearisation assumes; solving the exact
problem there costs the $\sqrt{\cdot}$-law $\sim$3\% and, unlike it, abandons $\sim$13\% of entries
(fast-changing, unimportant ones) entirely.

\begin{proposition}[$\sqrt{\cdot}$-law attention]\label{prop:sqrt}
Minimising $S\approx\sum_i w_i\lambda_i/(2f_i)$ s.t.\ $\sum_i f_i=K$ yields
$f_i^\star\propto\sqrt{w_i\lambda_i}$, with optimum $S^\star=(\sum_i\sqrt{w_i\lambda_i})^2/(2K)$.
\end{proposition}
\begin{proof}
$\partial_{f_i}[\sum_j w_j\lambda_j/(2f_j)+\mu(\sum_j f_j-K)]=-w_i\lambda_i/(2f_i^2)+\mu=0$ gives
$f_i\propto\sqrt{w_i\lambda_i}$; normalising by $K$ yields $S^\star$.
\end{proof}

The fractional $f_i$ are realised by an accumulator: each step every element accrues $f_i$ credit and
the $K$ highest-credit elements are re-observed, each paying one credit back, so element $i$ is picked at
long-run rate $f_i$ (capped at one look per step, so any $f_i{>}1$ saturates). The schedule thus needs the \emph{per-instance} rate $\lambda_i$, which must be estimated.
With an estimate $\hat\lambda_i$ the policy uses $f_i\propto\sqrt{w_i\hat\lambda_i}$; writing
$a_i=\sqrt{w_i\hat\lambda_i}$, $b_i=\sqrt{w_i\lambda_i}$, its realised cost is
$S(\hat\lambda)=\tfrac1{2K}(\sum_i a_i)(\sum_i b_i^2/a_i)$.

\begin{proposition}[Attention loss from rate error]\label{prop:cs}
$S(\hat\lambda)\ge S^\star$, with equality iff $\hat\lambda_i\propto\lambda_i$; the excess
$(\sum_i a_i)(\sum_i b_i^2/a_i)/(\sum_i b_i)^2$ is a Cauchy--Schwarz defect that grows as
$\hat\lambda$ diverges from $\lambda$.
\end{proposition}
\begin{proof}
Cauchy--Schwarz gives $(\sum_i a_i)(\sum_i b_i^2/a_i)\ge(\sum_i b_i)^2$, equality iff
$a_i\propto b_i$, i.e.\ $\hat\lambda_i\propto\lambda_i$.
\end{proof}

A memoryless VLM baseline is the extreme case of a mis-estimated rate: lacking per-instance history,
it assumes a homogeneous category rate and therefore scans \emph{uniformly}. This yields an especially
clean reading of the gap.
\begin{corollary}[Freshness gap $\propto$ root-volatility variance]\label{cor:var}
For uniform importance and the special case of a \emph{homogeneous} prior -- one rate for every
instance, so the baseline scans uniformly ($f_i{=}K/N$) -- it incurs
$\mathrm{Error}_{\mathrm{prior}}\propto\mathbb{E}[\lambda]$, while the $\sqrt{\cdot}$-law optimum
(Prop.~\ref{prop:sqrt}) incurs $\mathrm{Error}_{\mathrm{Map}}\propto(\mathbb{E}\sqrt{\lambda})^2$. Their
gap is proportional to the variance of the root-volatility,
\begin{equation}\label{eq:var}
\mathrm{Error}_{\mathrm{prior}}-\mathrm{Error}_{\mathrm{Map}}\ \propto\
\mathbb{E}[\lambda]-(\mathbb{E}\sqrt{\lambda})^2=\mathrm{Var}\!\big(\sqrt{\lambda}\big)\ \ge\ 0 .
\end{equation}
\end{corollary}
\begin{proof}
Uniform $f_i{=}K/N$ gives $\sum_i\lambda_i/f_i=(N/K)\sum_i\lambda_i\propto\mathbb{E}[\lambda]$;
Prop.~\ref{prop:sqrt} gives $S^\star=(\sum_i\sqrt{\lambda_i})^2/(2K)\propto(\mathbb{E}\sqrt{\lambda})^2$.
Subtracting and using $\mathrm{Var}(\sqrt\lambda)=\mathbb{E}[(\sqrt\lambda)^2]-(\mathbb{E}\sqrt\lambda)^2
=\mathbb{E}[\lambda]-(\mathbb{E}\sqrt\lambda)^2$ gives \eqref{eq:var}; non-negativity is Jensen.
\end{proof}
Eq.~\eqref{eq:var} is the interpretable form of Prop.~\ref{prop:cs}: the value of memory equals the
\emph{heterogeneity} of the scene's dynamics, measured as $\mathrm{Var}(\sqrt\lambda)$. It vanishes
when all instances share a rate and grows with that spread. With importance weights the same argument
gives a gap in $\mathrm{Var}(\sqrt{w\lambda})$. The baseline here scans \emph{uniformly}; against the
category-prior schedule we evaluate, the defect is no longer a function of
$\mathrm{Var}(\sqrt\lambda)$ alone --- it depends on the joint distribution of rates, estimates and
weights --- though empirically it tracks that scalar with an offset set by the prior's own
(\S\ref{sec:gap}). One caveat: the baseline models a VLM returning the
category \emph{mean} (a real VLM may do better or worse). The $\sqrt{\cdot}$-law's own suboptimality --
the source of the $h{=}0$ floor -- is taken up with the Whittle-index reference.

\textbf{Where the rate estimate comes from.} Two sources compete. An \emph{on-demand VLM} returns the
category movability prior $\rho_i$ -- one value per object \emph{type}. The \emph{map's memory}
accumulates each element's observed change history and forms a shrunken rate
$\hat\lambda_i=(m_i+\kappa\lambda_{\max}\rho_i)/(t_i+\kappa)$, where $m_i$ counts the
\emph{re-observations that found the entry changed} over $t_i$ observed steps (pseudo-count $\kappa$, in
rate units). Reporting only \emph{whether} an entry changed censors multiple changes between looks, so
$\hat\lambda_i$ is biased low the faster the entry is, compressing the estimated spread and making the
effect below \emph{conservative}. We report a \emph{held-out} estimate throughout and
separate the memory-attributable gain from the schedule-level floor.

The scheduler consumes two of the map's fields --- the category
movability prior $\rho_i$ and the per-instance change history --- and not the observed-motion score
or the confidence channel; and it decides \emph{which} entries to re-check, assuming a downstream
viewpoint or navigation module attempts the observation. Travel cost and viewpoint co-visibility are
outside the scheduling layer, and \S\ref{sec:spatial} is where they are measured.

\textbf{Heterogeneity is what memory buys.} By Prop.~\ref{prop:cs}, the prior is exact only when true
rates match the category norm; its loss grows with per-instance deviation. We model this with a knob
$h\in[0,1]$: $\lambda_i=\lambda_{\max}[(1-h)\rho_i+h\,u_i]$, $u_i\!\sim\!\mathrm{U}(0,1)$. At $h{=}0$
every instance equals its category prior (memory has nothing to add); at $h{=}1$ instance rates are
independent of category (only history can know them). Real scenes lie in between -- two chairs of the
same type differ in how often they are used -- so the memory's value is an empirical question of how
much heterogeneity a deployment holds. This knob is not a pure intervention: since
$\mathbb{E}[\rho]{=}0.55\neq\mathbb{E}[u]$, raising $h$ also drifts the mean rate by $5\%$. Repeating
the sweep with a mean-preserving construction (zero-mean perturbations about each object's own prior)
leaves the advantage at $1.7, 1.9, 2.4, 3.5, 6.2\%$ against $1.7, 2.0, 2.3, 3.6, 6.7\%$ in the same
$16$-seed control (Table~\ref{tab:main} reports the $40$-seed run), so the effect
is heterogeneity, not the drift.

\begin{figure*}[t]
\centering
\includegraphics[width=0.98\textwidth, trim=0 0.55cm 0 0, clip]{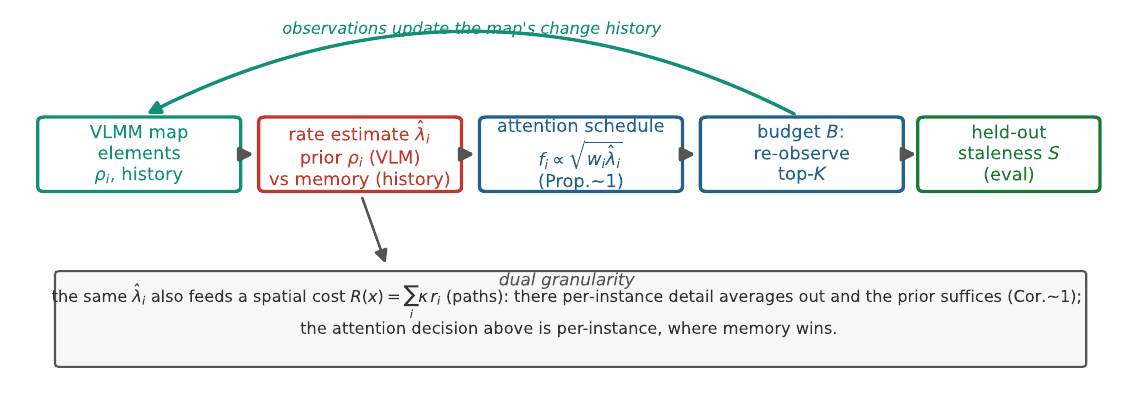}
\caption{\textbf{Budgeted re-perception as attention.} The rate estimate comes from the category prior
\emph{or} from observed history; the $\sqrt{\cdot}$-law sets each re-check frequency and the budget
picks $K$ per step. \emph{Dual granularity} (bottom): the same estimate feeds a spatial cost, where
per-instance detail averages out.}
\label{fig:method}
\end{figure*}

\section{Results and Discussion}\label{sec:results}
\textbf{Setup.} We use AI2-THOR~\cite{kolve2017ai2thor}. For the attention study we pool $951$ movable
object instances from $40$ scenes, each with its category movability prior; parameters
$\lambda_{\max}{=}0.12$, horizon $T{=}400$, budget $K{=}6\%$ of $N$ per step unless swept, importance
$w_i$ log-normal. Seed counts are $40$ (Table~\ref{tab:main}), $24$ (Tables~\ref{tab:whittle} and
\ref{tab:sens}) and $16$ (Fig.~\ref{fig:baselines} and the CLIP arm); $t$-statistics are paired over
seeds within a fixed object pool except for the fetch task, which is clustered over $25$ scenes.
To isolate the value of memory cleanly, we use a \emph{held-out} estimate: the change
history is accumulated on one realisation of the process, frozen, and applied as a static schedule to an
\emph{independent} realisation -- so the memory policy differs from the prior only in the rate estimate,
with no online adaptation and no scoring on the data it was fit to. We report the memory-attributable
reduction: the held-out advantage over the prior at heterogeneity $h$ minus its value at $h{=}0$ (a
schedule-level floor that is not memory).

\subsection{Why not path-shaping: the spatial-navigation negative}\label{sec:spatial}
We first report the spatial use that motivated the shift. Encoding VLMM's fields as a
planner cost and routing to minimise expected disruptions cuts that objective $\sim$35\% over $28$
simulated scenes. But the objective is close to what the cost minimises; on an \emph{independent} closed-loop
execution (objects actually relocating, counting real replans) the reduction is only $\sim$4\% at
$+14\%$ travel, and the map's observed history does no better than -- indeed slightly worse than -- the
category prior ($4.6$ vs.\ $4.3$ replans). The reason is structural: routing aggregates cost over a
\emph{region}, and summation averages out per-instance deviation, so a coarse prior suffices. The
per-instance decision -- \emph{which} element to attend to -- is where the prior's bias is exposed and
the memory pays off, consistent with Prop.~\ref{prop:cs}. This refines rather than contradicts the
case for behavior annotation~\cite{vlmm2026}: the movability/motion attribute is evaluated there as a
map-quality attribute; its \emph{downstream} payoff, we find, lies in re-perception, not in path-shaping
-- a planning-time proxy improves $\sim$35\% but the closed-loop gap is small.

\subsection{Memory directs attention better than a category prior}
Over $951$ objects (Table~\ref{tab:main}), the held-out history lowers value-weighted staleness below
the category prior by a margin that grows with heterogeneity: $1.7,1.9,2.3,3.4,6.5\%$ at
$h{=}0,0.25,0.5,0.75,1.0$. Two caveats bound this. First, $\sim$1.7\% persists at $h{=}0$,
where there is \emph{no} per-instance variation to learn; this floor is not memory but a suboptimality
of the $\sqrt{\cdot}$-schedule on the discrete problem (a noisy estimate can out-schedule the true-rate
one). The \emph{memory-attributable} effect is therefore the heterogeneity-driven increase -- up to
$\sim$5 points at $h{=}1$, tapering to nothing when scenes are homogeneous. Second, a naive
\emph{online} estimator lets an arm adapt on the realisation it is scored on; we therefore report the
held-out protocol throughout, which is the conservative choice here --- it \emph{raises} the measured
advantage by $0.7$--$1.1$ points at every $h$ rather than lowering it. Net: the map's memory earns its place
in the attention policy when the decision is per-instance \emph{and} instances depart from their
category norm.

\textbf{Memory protects the objects that matter.} The $\sqrt{w_i\lambda_i}$ schedule \emph{reallocates}
the budget toward important entries, so memory's benefit is concentrated on them. Restricted to the
top-$10\%$-importance objects, the memory-attributable reduction (net of the $h{=}0$ floor) reaches
$+7.7\%$ at $h{=}1$, versus $+4.9\%$ on the mean. These are percentages of different baselines
(the important objects are already fresher), so the ratio is not itself a measure of reallocation:
in absolute staleness the important objects gain more, but once the $h{=}0$ floor is removed they
gain slightly less than the pool.
(We report the floor-subtracted figure: the raw top-$10\%$ gap is $\sim$11\% already at $h{=}0$, where
there is no per-instance information, so it is the schedule artifact of Table~\ref{tab:whittle}, not
memory.) Attention \emph{is} this reallocation -- memory trades freshness on churny unimportant objects
to protect the fragile, task-relevant ones.

\begin{table}[t]
\centering
\caption{Held-out staleness and memory advantage by heterogeneity $h$ ($N{=}951$).}
\label{tab:main}
\footnotesize\setlength{\tabcolsep}{5pt}
\begin{tabular}{lrrrrr}
\toprule
heterogeneity $h$ & $0$ & $0.25$ & $0.5$ & $0.75$ & $1.0$\\
\midrule
prior (category) & $.294$ & $.290$ & $.284$ & $.275$ & $.263$\\
memory (history)      & $.289$ & $.285$ & $.277$ & $.265$ & $.246$\\
advantage (\%)        & $1.7$ & $1.9$ & $2.3$ & $3.4$ & $6.5$\\
memory-attributable (\%) & $0.0$ & $0.2$ & $0.6$ & $1.7$ & $4.9$\\
\bottomrule
\end{tabular}
\end{table}

\subsection{Against competitive schedulers: the prior is the weak link}\label{sec:baselines}
So far memory beat a category prior -- but that prior and the memory policy share one $\sqrt{\cdot}$-law
schedule, so the comparison isolates the rate estimate, not a deployable system. We add the freshness
schedulers a practitioner would actually build and run them all \emph{online} (each adaptive, fair to
each other): a memory-light \emph{oldest-first} (observe the $K$ highest $w_i\!\cdot\!\text{age}_i$, no
rate model), weighted round-robin, and a model-free \emph{Thompson-sampling} learner (Gamma posterior
on $\lambda_i$). Two things stand out (Fig.~\ref{fig:baselines}). First, the category prior is the
\emph{weak link}: even memory-light oldest-first beats it at every $h$, and weighted round-robin ---
which uses no map at all --- is within $0.8$ points of it and overtakes it at $h{=}1$. Second, the
\emph{history}-based policies win: Thompson tracks memory, and memory matches the true-rate schedule ($0.2483$
vs.\ $0.2478$ at $h{=}1$; the latter is not an upper bound, and held-out memory in fact edges it at
every $h$). Memory's edge over the best memory-light scheduler is smaller and
conditional -- $+0.1$ to $+3.8\%$ over oldest-first under skewed importance, and it \emph{loses} to
oldest-first under uniform importance until $h{=}1$. In sum: observed history yields the best
re-perception schedule we test, matching the true-rate $\sqrt{\cdot}$ schedule to within $0.2\%$
(the two are not separated by a paired test); the category prior is the weakest arm here, though
Table~\ref{tab:whittle} shows it ahead of memory at low heterogeneity under the Whittle index; and
memory's specific advantage over trivial age-scheduling emerges when importance is skewed and
rates are heterogeneous.

\begin{figure}[tb]
\centering
\includegraphics[width=0.76\columnwidth,trim=0 0.45cm 0 0,]{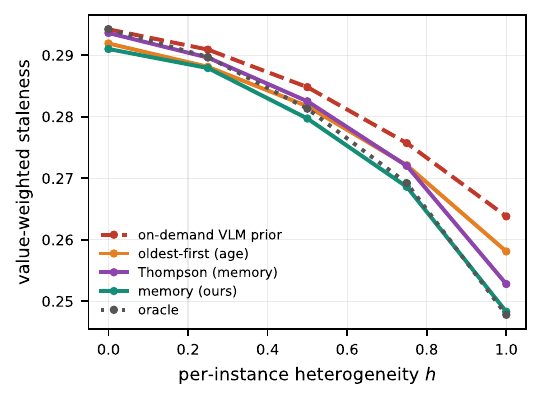}
\caption{Competitive schedulers, run online (skewed importance). Even memory-light oldest-first beats
the category prior (red, dashed); the history-based policies win.}
\label{fig:baselines}
\end{figure}

\subsection{A restless-bandit index removes the schedule floor}\label{sec:whittle}
Budgeted re-perception is a \emph{restless multi-armed bandit}~\cite{whittle1988restless}: each entry is
an arm whose state is its age $\tau$ (steps since last seen), with holding cost
$h(\tau){=}1{-}(1{-}\lambda)^\tau$ -- exact for the per-step change process we simulate, of which
\eqref{eq:stale} is the continuous approximation to $O(\lambda^2)$ -- and a reset on observation. For this deterministic-age bandit the
Whittle index is closed-form,
\begin{equation}\label{eq:whittle}
W_i(\tau)=\sum_{s=0}^{\tau-1}\!\big[h(\tau){-}h(s)\big]
 = w_i\Big[\tfrac{1-(1-\lambda_i)^{\tau}}{\lambda_i}-\tau(1-\lambda_i)^{\tau}\Big],
\end{equation}
increasing in age, and in rate only while an entry is fresh: $W_i\!\to\!w_i/\lambda_i$, so past
$\tau\!\approx\!17$ steps it prefers the \emph{slower} entry, whose reset buys more freshness. The policy
re-observes the $K$ entries of highest $W_i(\tau_i)$.
Two consequences (Table~\ref{tab:whittle}). First, the Whittle policy beats the $\sqrt{\cdot}$-law
($4.1\%$ lower staleness with true rates), confirming the $\sqrt{\cdot}$-law is a heuristic. Most of
that gap is the linearisation of \eqref{eq:stale} rather than discreteness: the schedule runs at a
median $\lambda_i\tau_i\!\approx\!1$ (interquartile $0.6$--$1.6$), where the exact continuous optimum
already recovers $3.2$ of the $4.1$ points, falling to $1.2$ at $\lambda_i\tau_i{=}0.6$. Second -- and this \emph{removes} the $h{=}0$ floor above -- under the proper index a noisy history estimate no longer out-schedules the
true-rate policy: memory sits $1.2$--$2.5\%$ \emph{above} the Whittle oracle at every $h$. These arms are scored
\emph{online}, the conservative direction here (held-out: $-0.8\%$ to $+4.9\%$).
The core claim survives the better scheduler, and cleaner: memory beats the prior only where
heterogeneity is real ($-1.2\%$ at $h{=}0$ rising to $+4.5\%$ at $h{=}1$), with no floor to subtract --
the crossover \emph{is} the memory-attributable effect.

\begin{table}[t]
\centering
\caption{Whittle-index scheduler ($N{=}951$): value-weighted staleness.}
\label{tab:whittle}
\footnotesize\setlength{\tabcolsep}{4pt}
\begin{tabular}{lrrrrr}
\toprule
heterogeneity $h$ & $0$ & $0.25$ & $0.5$ & $0.75$ & $1.0$\\
\midrule
$\sqrt{\cdot}$-law oracle   & $.295$ & $.290$ & $.282$ & $.269$ & $.248$\\
Whittle oracle              & $.282$ & $.278$ & $.271$ & $.259$ & $.237$\\
Whittle prior (VLM)         & $.282$ & $.279$ & $.274$ & $.266$ & $.254$\\
Whittle memory              & $.285$ & $.282$ & $.275$ & $.264$ & $.243$\\
\midrule
memory vs.\ prior (\%)      & $-1.2$ & $-0.9$ & $-0.3$ & $+0.9$ & $+4.5$\\
\bottomrule
\end{tabular}
\end{table}

The effect is not an artifact of the Poisson model. Re-running the held-out test under three other
processes with identical per-object mean rates -- \emph{bursty} (Markov-modulated), \emph{correlated}
(self-exciting), and \emph{diurnal} (periodic) -- the memory-attributable reduction stays positive and
significant ($+1.7\%$ Poisson, $+0.9$ bursty, $+2.2$ correlated, $+1.5$ diurnal; all $t{>}3.8$),
attenuated under burstiness because clustered
changes make per-instance rates harder to estimate. Across a sensitivity sweep (Table~\ref{tab:sens})
it grows with observation horizon ($+0.9\!\to\!+2.0\%$ from $100$ to $800$ steps) and importance skew
($+1.2\%$ uniform to $+2.9\%$), is robust to prior strength, and is \emph{budget-gated}: near zero at a
tight $3\%$ budget, rising to $+5.4\%$ at $20\%$ -- memory helps once the budget is large enough to act
on per-instance differences.

\textbf{Observation noise.} The estimates so far assume perfect re-perception; real observation misses
or hallucinates changes. Flipping each observed-change bit with probability $\epsilon$, the memory
advantage is robust to moderate noise -- at $h{=}1$ it is essentially intact to $\epsilon{=}0.2$
($+4.2\%$, from $+4.9\%$ clean) and still positive at $\epsilon{=}0.3$ ($+2.6\%$) -- but degrades to a
\emph{loss} once observations are near-random ($\epsilon{=}0.5$: $-3.2\%$, worse than the prior), as a
corrupted history is worse than a category average. A deployed policy should therefore fall back to the
prior when observation reliability is low. For context, under skewed importance a no-map uniform schedule is worse
throughout (staleness $\sim$0.34--0.39 vs.\ $\sim$0.26 for the prior), though under \emph{uniform}
importance the ordering reverses and the prior trails a uniform schedule at every $h$: a category prior already captures
most of the achievable freshness, and memory supplies the per-instance increment a prior structurally
cannot.

\begin{table}[t]
\centering
\caption{Sensitivity of the memory-attributable reduction (\%, held-out, Poisson, $N{=}951$).}
\label{tab:sens}
\footnotesize\setlength{\tabcolsep}{4pt}
\begin{tabular}{llrrrr}
\toprule
knob & values & \multicolumn{4}{c}{memory-attributable reduction}\\
\midrule
budget $K$        & $3/6/12/20\%$      & $-0.1$ & $+1.7$ & $+4.2$ & $+5.4$\\
obs.\ horizon     & $100/200/400/800$  & $+0.9$ & $+1.6$ & $+1.7$ & $+2.0$\\
importance skew   & $0/0.5/1/2$        & $+1.2$ & $+1.3$ & $+1.7$ & $+2.9$\\
prior pseudocount & $1/3/10/30$        & $+1.6$ & $+1.7$ & $+1.9$ & $+2.0$\\
\bottomrule
\end{tabular}
\end{table}

\subsection{The freshness gap tracks root-volatility variance}\label{sec:gap}
Corollary~\ref{cor:var} predicts the memory advantage should be governed by a single scalar --
$\mathrm{Var}(\sqrt\lambda)$, the heterogeneity of the scene's root-volatility. We test this directly, regressing the realised
held-out memory reduction on the measured $\mathrm{Var}(\sqrt\lambda)$ over $11$ heterogeneity levels. The
relationship is \emph{linear} ($R^2{=}0.99$), consistent with the corollary's form. We claim no more
than that: Cor.~\ref{cor:var} is proved for uniform importance and a homogeneous prior, and for a
general prior the Cauchy--Schwarz defect depends on the joint distribution of rates, estimates and
weights rather than on this one scalar. The sweep is a single-parameter family, so it cannot identify
the exponent, and over $h\!\le\!0.2$ neither $\mathrm{Var}(\sqrt\lambda)$ nor the gap moves
monotonically in $h$ -- both sit within noise of zero, consistent with the threshold below. Moreover, the fit's
$x$-intercept ($0.0008$) coincides with the prior's own $\mathrm{Var}(\sqrt\lambda)$ ($0.0009$): memory
pays only for heterogeneity the category prior cannot already explain, and adds nothing below that
threshold. How spread out the instance volatilities are thus predicts, quantitatively, when the map's
memory is worth its perception budget.

\subsection{A real vision-language prior (AI2-THOR $+$ CLIP)}\label{sec:clip}
The on-demand VLM has so far been \emph{modelled} by a category movability table. We now replace it with
a genuine one: over $322$ objects across $8$ scenes we render each object from AI2-THOR and compute a
CLIP~\cite{radford2021clip} zero-shot movability prior from its image (ViT-B/16,
\texttt{laion2b\_s34b\_b88k}; the two prompts ``a movable household object'' / ``a fixed built-in
fixture'', softmax over $100\times$ cosine, on the largest crop of at least $300$\,px; prompts were not
tuned).
CLIP is a \emph{weak} movability estimator -- its prior correlates only $r{=}0.40$ with ground-truth
movability -- confirming that appearance is a poor cue for how often an object \emph{actually} moves.
Against this real VLM prior, the map's observed-change memory wins decisively (held-out): staleness
reduction $+21.4, 23.9, 24.7, 25.1, 26.3\%$ at $h{=}0,0.25,0.5,0.75,1.0$ (all $t{>}23$ over $16$
seeds). The result has two mechanisms, which we keep separate. The $\sim$21\% base is
memory correcting the VLM prior, but most of that is the prior's \emph{scale}, not its accuracy: the
$100\times$ softmax saturates $\rho_{\text{clip}}$ (dispersion $\sigma/\mu{=}1.01$ against $0.75$ for
true movability, with $16.5\%$ of objects at ${\approx}0$ and $6.5\%$ at ${\approx}1$), and the
$\sqrt{\cdot}$-law is sensitive to spread. Rescaling the \emph{same} CLIP scores to the mean and spread of true
movability -- an affine map, so the ranking and the correlation $r{=}0.40$ are both exactly
unchanged -- leaves $+7.3, 9.0, 12.9\%$ at $h{=}0,0.5,1$ (matching the full movability
\emph{distribution} instead gives $+8.1, 9.5, 13.1\%$). So a better-\emph{calibrated} readout shrinks this term; a
stronger VLM need not. The
$+4.8\%$ rise from $h{=}0$ to $h{=}1$ is the \emph{per-instance heterogeneity} effect: the
VLM-independent core claim, since no appearance model can see how often a specific instance is used.
We report the part that cuts against us too: against a \emph{clean} category prior on the same objects
(an idealised VLM a real robot lacks -- it only sees appearance), memory is \emph{worse} at low
heterogeneity ($-9\%$ at $h{=}0$, $-2.6\%$ at $h{=}0.25$) and wins only for $h\!\ge\!0.5$. A controlled
prior-quality sweep (Fig.~\ref{fig:prior}) makes the decomposition precise: as a category prior's
accuracy $r$ (correlation with true movability) rises from $0.2$ to $0.99$, the base ($h{=}0$)
advantage shrinks from $+9\%$ to $+3\%$, while the heterogeneity slope ($\sim$5\%) is roughly constant
in $r$. The sweep varies correlation at fixed calibration on the $951$-object pool, so it does not place CLIP:
its $+21\%$ at the same $r{=}0.40$ reflects mis-calibrated \emph{scale}, which the $\sqrt{\cdot}$-law
also sees. The VLM-independent contribution is that
slope. Object appearance and the VLM prior are real here; the change process and re-observation remain
simulated (natural heterogeneity would need real longitudinal data).

\begin{figure}[tb]
\centering
\includegraphics[width=0.76\columnwidth,trim=0 0.35cm 0 0,]{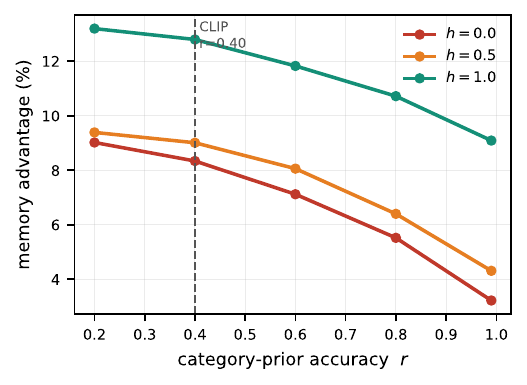}
\caption{Prior-quality sweep. A better category prior shrinks the base ($h{=}0$) advantage, but the
heterogeneity slope is roughly constant in $r$. Correlation alone does not place CLIP on this axis.}
\label{fig:prior}
\end{figure}

\subsection{Downstream fetch task: does fresher memory help the robot?}\label{sec:fetch}
Staleness is an intrinsic proxy; we put it to a task. In $25$ AI2-THOR scenes the robot runs an
object-fetch loop under the budget: objects relocate at $\lambda_i$, the robot re-observes $K$ entries
per its schedule (updating remembered locations), and requests arrive for object $j\!\sim\!w_j$; a
request to a \emph{stale} entry is a wasted trip, costing the extra travel to the true location -- a
geometric cost the staleness objective never sees. The result confirms the effect and sharpens the
reframe. Memory-scheduling incurs fewer wasted trips than the category prior -- $+4.9\%$ at $h{=}1$, matched in
excess travel ($+4.4\%$, $t{=}4.4$ over $25$ scenes $\times$ $6$ seeds) -- so a fresher map does translate to a cheaper task.
But at moderate heterogeneity ($h{=}0.5$) memory ties the prior ($+0.3\%$, n.s.), and the memory-light
\emph{oldest-first} recency heuristic matches or beats memory on the task ($-4.2\%$ at $h{=}0.5$, tie at
$h{=}1$). The lesson: what beats the \emph{memoryless} category prior is using the map's
observations at all -- as change-history \emph{or} simply as recency of last sighting; the persistent
map earns its keep, but on this generic, importance-random fetch task sophisticated per-instance rate
modelling is not needed over cheap recency; the rate channel's distinct value re-emerges under language
conditioning ($\S$\ref{sec:langattn}), where it beats a relevance-weighted recency baseline
($+2.9\%$ over $26$ stratified queries, at full heterogeneity).

\textbf{Rendered perception.} The observation model so far is an abstract bit-flip; we now drive it
from detection on rendered AI2-THOR views. Over
$138$ objects in $8$ scenes we measure each object's \emph{detectability} -- rendering the scene from
many agent poses and running instance detection, then taking the fraction of viewpoints from which the
object is detected with a usable crop. It is low and heavy-tailed (mean $0.15$ from random viewpoints;
small or occluded objects are rarely seen), and we run the re-perception loop with each observation
succeeding at its real rate. The memory advantage turns out to be \emph{conditional on observation
reliability}: at mean success $0.15, 0.30, 0.48, 0.60, 0.67$ it reads
$-2.8, +0.4\,\text{(n.s.)}, +2.2, +1.7\,\text{(n.s.)}, +2.4\%$ ($138$ objects, $h{=}1$). These are raw
advantages; subtracting the $h{=}0$ floor as elsewhere leaves $-2.5$ to $+0.8\%$, so under degraded
observation the memory-attributable part is close to zero -- so a robot that navigates to a usable viewpoint gains, while at the pessimistic
random-viewpoint mean of $0.15$ the advantage \emph{reverses}. The mechanism is
structural, and matches the observation-noise sweep: memory's rate estimate is only as good as its
observations, so when perception frequently fails a fixed category prior wins. Observed history helps
for re-perception \emph{provided the robot can actually re-perceive}.

\subsection{Language conditions what to keep fresh -- the map's distinctive use}\label{sec:langattn}
The results so far use a fixed importance $w_i$. But the point of an \emph{open-vocabulary} map is that
a task can \emph{name} what matters. We instantiate this: a language instruction (``a cup'', ``something
to sit on'') sets each element's importance by its CLIP relevance
$w_i{=}s_i(q)$: the similarity $\mathbf{f}_i^{\!\top}\mathrm{CLIP}_{\mathrm{txt}}(q)$, scene-centred,
clipped at zero and rescaled to $(0,1]$ with a small floor so $\sqrt{w_i\hat\lambda_i}$ is real. The
schedule keeps the
\emph{relevant} objects fresh. Relevance and change rate are \emph{independent} (measured correlation
$\approx 0$), so a good schedule needs \emph{both}. Over $159$ objects and $26$ queries
(Fig.~\ref{fig:langattn}), a map using both -- open-vocabulary relevance \emph{and} observed dynamics
($f_i\propto\sqrt{s_i(q)\,\hat\lambda_i}$) -- achieves the lowest relevance-weighted staleness. The
\emph{decisive} baseline is \emph{relevance-weighted oldest-first} ($s_i(q)\!\cdot\!\text{age}_i$) -- a
static open-vocabulary map plus a last-seen timestamp, with \emph{no} rate channel; it is the strongest
competitor (recency is a strong policy, as the fetch task showed). VLMM beats it by $+2.9\%$ over $26$ queries
stratified into categories, affordances, attributes, task phrases and paraphrases
($95\%$ CI $[2.5, 3.2]$, $26/26$ queries, every stratum between $+2.4$ and $+3.6\%$, $8$ seeds each).
The margin is robust to wording but not to \emph{heterogeneity}: this arena draws instance rates
independently of the category prior, i.e.\ $h{=}1$, and at $h{=}0.5$ and $h{=}0$ the ordering reverses
($-3.1\%$ and $-4.0\%$, losing on every query). A true-rate schedule gains only $+0.5\%$ over the same
baseline, so at $h{=}1$ the margin is largely the schedule-level effect of \S\ref{sec:results} rather
than rate information. The honest claim is therefore narrow: when instance rates depart fully from
category norms, the motion channel adds value beyond a timestamp; when they do not, recency wins. Dropping language grounding entirely (recency, no query) is $+39\%$ worse -- establishing the
language channel is necessary -- and a category prior (relevance $\times$ prior) is $+9.2\%$
worse. So language-conditioned re-perception needs open-vocabulary grounding \emph{and} per-instance
dynamics together -- a task a static open-vocabulary map and a plain dynamics map both fail. (The
held-out protocol of $\S$\ref{sec:results} gives the same ordering with larger gaps; we report the more
conservative online arena, in which recency is the strongest baseline.)

\begin{figure}[tb]
\centering
\includegraphics[width=0.70\columnwidth,trim=0 0.35cm 0 0,]{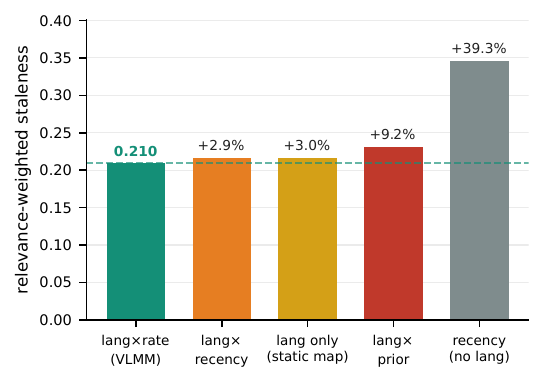}
\caption{Language-conditioned re-perception over the $26$ stratified queries ($N{=}159$, $8$ seeds,
$h{=}1$; online arena, relevance-weighted staleness, lower better). Using both channels
(lang$\times$rate) beats every ablation, including the strongest competitor, relevance-weighted
oldest-first ($+2.9\%$, $26/26$ queries).}
\label{fig:langattn}
\end{figure}

\section{Conclusions}
A behavior-annotated map is asked two questions and answers them differently. To ``how do I walk around
this room?'' its memory adds little. To ``what should I attend to right now?'' -- which element to
re-observe, under a budget -- the map's memory of what has actually been changing beats a
\emph{memoryless} category prior \emph{once instance rates depart from category norms}, tracking
$\mathrm{Var}(\sqrt\lambda)$ (Cor.~\ref{cor:var}); a Whittle-index reference sharpens this into a
clean crossover, the prior leading below it; told \emph{what} to keep fresh, its two channels together
beat a static open-vocabulary map, a category prior, and a relevance-weighted recency heuristic. Three
conditions bound this. The rate channel's edge over cheap \emph{recency} is setting-dependent --
negligible on the generic fetch task, small but significant under language conditioning -- so the robust
claim is that observation-based scheduling beats the memoryless prior. The advantage is \emph{conditional on observation reliability}, reversing below a
$\sim$0.3 success rate under measured AI2-THOR detectability, and the change process is simulated with a
parametric heterogeneity $h$ whose real-world value we do not measure -- a live VLM \emph{and} measured
longitudinal dynamics together are the experiment left to future work. The map's memory is a tool for
resource-allocation planning, not spatial-navigation planning: it tells the robot what to pay attention
to.

\bibliographystyle{IEEEtran}
\bibliography{refs}

\begin{thebibliography}{10}
\providecommand{\url}[1]{#1}
\csname url@rmstyle\endcsname
\providecommand{\newblock}{\relax}
\providecommand{\bibinfo}[2]{#2}
\providecommand\BIBentrySTDinterwordspacing{\spaceskip=0pt\relax}
\providecommand\BIBentryALTinterwordstretchfactor{4}
\providecommand\BIBentryALTinterwordspacing{\spaceskip=\fontdimen2\font plus
\BIBentryALTinterwordstretchfactor\fontdimen3\font minus
  \fontdimen4\font\relax}
\providecommand\BIBforeignlanguage[2]{{%
\expandafter\ifx\csname l@#1\endcsname\relax
\typeout{** WARNING: IEEEtran.bst: No hyphenation pattern has been}%
\typeout{** loaded for the language `#1'. Using the pattern for}%
\typeout{** the default language instead.}%
\else
\language=\csname l@#1\endcsname
\fi
#2}}

\bibitem{vlmm2026}
D.~Ghosh and A.~Shakya, ``{Vision-Language-Motion Maps}: An open-vocabulary,
  uncertainty-aware, queryable motion attribute for {3D} scene maps,''
  \emph{arXiv:2607.16173}, 2026.

\bibitem{huang2023vlmaps}
C.~Huang, O.~Mees, A.~Zeng, and W.~Burgard, ``Visual language maps for robot
  navigation,'' in \emph{Proc. IEEE ICRA}, 2023.

\bibitem{jatavallabhula2023conceptfusion}
K.~M. Jatavallabhula \emph{et~al.}, ``Conceptfusion: Open-set multimodal 3d
  mapping,'' in \emph{Proc. RSS}, 2023.

\bibitem{gu2024conceptgraphs}
Q.~Gu \emph{et~al.}, ``Conceptgraphs: Open-vocabulary 3d scene graphs for
  perception and planning,'' in \emph{Proc. IEEE ICRA}, 2024.

\bibitem{werby2024hovsg}
A.~Werby, C.~Huang, M.~B{\"u}chner, A.~Valada, and W.~Burgard, ``Hierarchical
  open-vocabulary 3d scene graphs for language-grounded robot navigation,'' in
  \emph{Proc. RSS}, 2024.

\bibitem{schmid2024khronos}
L.~Schmid \emph{et~al.}, ``Khronos: A unified approach for spatio-temporal
  metric-semantic slam in dynamic environments,'' in \emph{Proc. RSS}, 2024.

\bibitem{dewan2017deep}
A.~Dewan, G.~L. Oliveira, and W.~Burgard, ``Deep semantic classification for 3d
  lidar data,'' in \emph{Proc. IEEE/RSJ IROS}, 2017.

\bibitem{marder2010office}
E.~Marder-Eppstein, E.~Berger, T.~Foote, B.~Gerkey, and K.~Konolige, ``The
  office marathon: Robust navigation in an indoor office environment,'' in
  \emph{IEEE Int. Conf. on Robotics and Automation (ICRA)}, 2010, pp. 300--307.

\bibitem{dynamem2024}
P.~Liu \emph{et~al.}, ``Dynamem: Online dynamic spatio-semantic memory for open world
  mobile manipulation,'' \emph{arXiv:2411.04999}, 2024.

\bibitem{dovsg2025}
Z.~Yan \emph{et~al.}, ``Dynamic
  open-vocabulary 3d scene graphs for long-term language-guided mobile
  manipulation,'' \emph{arXiv:2410.11989}, 2025.

\bibitem{dualmap2025}
J.~Jiang, Y.~Zhu, Z.~Wu, and J.~Song, ``Dualmap: Online open-vocabulary
  semantic mapping for natural language navigation in dynamic changing
  scenes,'' \emph{arXiv:2506.01950}, 2025.

\bibitem{kurenkov2023scenegraphmemory}
A.~Kurenkov \emph{et~al.}, ``Modeling dynamic
  environments with scene graph memory,'' in \emph{Proc. ICML}, vol. 202, 2023.

\bibitem{bajcsy1988active}
R.~Bajcsy, ``Active perception,'' \emph{Proceedings of the IEEE}, vol.~76,
  no.~8, pp. 966--1005, 1988.

\bibitem{bajcsy2018revisiting}
R.~Bajcsy, Y.~Aloimonos, and J.~K. Tsotsos, ``Revisiting active perception,''
  \emph{Autonomous Robots}, vol.~42, no.~2, pp. 177--196, 2018.

\bibitem{connolly1985nbv}
C.~I. Connolly, ``The determination of next best views,'' in \emph{Proc. IEEE ICRA}, vol.~2, 1985,
  pp. 432--435.

\bibitem{bircher2016receding}
A.~Bircher, M.~Kamel, K.~Alexis, H.~Oleynikova, and R.~Siegwart, ``Receding
  horizon ``next-best-view'' planner for {3D} exploration,'' in \emph{Proc. IEEE ICRA}, 2016, pp.
  1462--1468.

\bibitem{smith2012persistent}
S.~L. Smith, M.~Schwager, and D.~Rus, ``Persistent robotic tasks: Monitoring
  and sweeping in changing environments,'' \emph{IEEE Transactions on
  Robotics}, vol.~28, no.~2, pp. 410--426, 2012.

\bibitem{alamdari2014persistent}
S.~Alamdari, E.~Fata, and S.~L. Smith, ``Persistent monitoring in discrete
  environments: Minimizing the maximum weighted latency between observations,''
  \emph{The International Journal of Robotics Research}, vol.~33, no.~1, pp.
  138--154, 2014.

\bibitem{cho2003refresh}
J.~Cho and H.~Garcia-Molina, ``Effective page refresh policies for web
  crawlers,'' \emph{ACM Transactions on Database Systems}, vol.~28, no.~4, pp.
  390--426, 2003.

\bibitem{whittle1988restless}
P.~Whittle, ``Restless bandits: Activity allocation in a changing world,''
  \emph{Journal of Applied Probability}, vol. 25A, pp. 287--298, 1988.

\bibitem{kadota2018scheduling}
I.~Kadota, A.~Sinha, E.~Uysal-Biyikoglu, R.~Singh, and E.~Modiano, ``Scheduling
  policies for minimizing age of information in broadcast wireless networks,''
  \emph{IEEE/ACM Transactions on Networking}, vol.~26, no.~6, pp. 2637--2650,
  2018.

\bibitem{yates2021aoi}
R.~D. Yates, Y.~Sun, D.~R. Brown, S.~K. Kaul, E.~Modiano, and S.~Ulukus, ``Age
  of information: An introduction and survey,'' \emph{IEEE J. Sel. Areas Commun.}, vol.~39, no.~5, pp. 1183--1210, 2021.

\bibitem{avrachenkov2020changerate}
K.~Avrachenkov, K.~Patil, and G.~Thoppe, ``Change rate estimation and optimal
  freshness in web page crawling,'' in \emph{Proc. VALUETOOLS}, 2020.

\bibitem{kolve2017ai2thor}
E.~Kolve \emph{et~al.}, ``{AI2-THOR}: An interactive 3d environment for visual
  ai,'' \emph{arXiv:1712.05474}, 2017.

\bibitem{radford2021clip}
A.~Radford \emph{et~al.}, ``Learning
  transferable visual models from natural language supervision,'' in
  \emph{Proc. ICML}, 2021.

\end{thebibliography}
\end{document}